\documentclass{amsart}
\usepackage{cite}
\usepackage{amsmath,amssymb,amsfonts}
\usepackage{algorithmic}
\usepackage{graphicx}
\usepackage{textcomp}
\usepackage[table,xcdraw]{xcolor}
\def\BibTeX{{\rm B\kern-.05em{\sc i\kern-.025em b}\kern-.08em
    T\kern-.1667em\lower.7ex\hbox{E}\kern-.125emX}}
\usepackage{float}
\usepackage{amsthm}
\usepackage{xypic}
\usepackage{cleveref}
\usepackage{bm}
\usepackage{adjustbox}

\newtheorem{theorem}{Theorem}[section]
\newtheorem{lemma}[theorem]{Lemma}

\newtheorem{corollary}[theorem]{Corollary}
\theoremstyle{definition}
\newtheorem{definition}[theorem]{Definition}
\newtheorem{remark}[theorem]{Remark}

\def\N{\mathbb{N}}

\def\R{\mathbb{R}}

\def\W{\mathbb{W}}

\newcommand\set[1]{\left\lbrace #1\right\rbrace}
\newcommand\norm[1]{\left\lVert#1\right\rVert}
\newcommand{\vect}[1]{\boldsymbol{#1}}
\newcommand{\abs}[1]{\left\vert #1 \right\vert}

\begin{document}

\title[Adaptive Template Systems]{Adaptive template systems: 
Data-driven  feature selection  
for  learning with  persistence diagrams 
}

\author[Luis Polanco]{Luis Polanco}
\address{
\shortstack[l]{
Department of Computational Mathematics, Science \& Engineering \\
Department of Mathematics, \\
Michigan State University \\
East Lansing, MI, USA.}}
\email{polanco2@msu.edu}

\author[Jose Perea]{Jose A. Perea }
\address{
\shortstack[l]{
Department of Computational Mathematics, Science \& Engineering \\
Department of Mathematics, \\
Michigan State University \\
East Lansing, MI, USA.}}
\email{joperea@msu.edu}
\thanks{This work was partially supported by the NSF (DMS-1622301)}


\subjclass[2010]{Primary 55N99, 68W05; Secondary 55U99}


\keywords{Topological Data Analysis, Persistent Homology, Machine Learning, Featurization.}

\maketitle

\begin{abstract}
    Feature extraction from persistence diagrams, 
    as a tool to enrich machine learning techniques, has received increasing attention in recent years. In this paper we explore an adaptive methodology to localize features in persistent diagrams, which are then used in learning tasks. Specifically, we investigate three algorithms, CDER, GMM and HDBSCAN, to obtain adaptive template functions/features. 
    Said features are evaluated in three classification experiments with persistence diagrams. Namely, manifold, human shapes and protein classification. 
    The main conclusion of our analysis is that adaptive template systems, as a feature extraction technique, yield competitive and often superior results in the studied examples. Moreover, from the adaptive algorithms here studied,    CDER  consistently provides the most reliable and robust adaptive featurization. 
\end{abstract}

\section{Introduction}
One of the central questions 
in Topological Data Analysis (TDA)
is how to leverage topological information, like persistence diagrams \cite{perea2018brief}, for  machine learning purposes. 
This idea has been explored, for instance, in \cite{landscapes, protein_class, Reininghaus, preprint} and \cite{persistent_images}. 

In particular,  \cite{preprint} establishes theoretical and computational tools to translate  supervised machine learning tasks (e.g., classification and regression) with topological features, into the problem of approximating continuous real-valued functions on the space of persistence diagrams, $\mathcal{D}$, endowed with the Bottleneck distance. 
The main concept is that of templates. 
These are continuous real-valued compactly supported functions on $\mathbb{W} := \set{(x_1,x_2)\in\mathbb{R}^2 \;\vert\; 0\leq x_1<x_2}$, which (by integration against  persistence 
measures, whereby a diagram is replaced by a sum of Dirac deltas) 
yield continuous functions on  
$\mathcal{D}$.
The same work shows that one can construct countable families of template functions (a template system), 
which in turn give rise  to 
dense subsets of $C(\mathcal{D},\R)$
with respect to the compact-open topology. \Cref{thm:regression} below indicates how a template system can be utilized to generate (polynomial) features for supervised  machine learning problems on persistence diagrams.

In this paper we address the question of producing template systems that 
are attuned (adaptive) to the input data set
and the supervised classification problem at hand.
We explore and compare different strategies to assemble adaptive template systems; 
namely, Cover-Tree Entropy Reduction (CDER) \cite{cder_paper}, Gaussian Mixture Models (GMM) \cite{gmm} and Hierarchical density-based spatial clustering of applications with noise (HDBSCAN) \cite{hdbscan}.
The conclusion is
that   CDER   is the most
consistently successful strategy out 
of the ones explored. 

We present three different examples where we use adaptive template functions to extract features   from persistence diagrams 
for supervised classification tasks.
First, we explore a $6$ class classification problem presented in \cite{preprint}. In this problem, several random samples are taken form each
of $6$ manifolds, and  persistence diagrams
  are computed as descriptors in each case. 
We then use our adaptive template functions and  compare to the 
results 
provided in \cite{preprint}. The average classification accuracy of adaptive templates for both the training and testing sets is comparable to that of \cite{preprint}. On the other hand, the standard deviation of our results is much smaller, making our methodology more stable compared to the template systems proposed in \cite{preprint}.

We then report results on the  SHREC 2014 synthetic data set \cite{Pickup:2014},  which involves  a $15$ class supervised learning problem. Each class in this data set corresponds to a human body in five different poses and three different shapes: male, female and child. The data points in this data set are 3D meshes. 
In \cite{Reininghaus}  a heat kernel signature is computed for each mesh and for $10$ different kernel parameter values.
This defines $10$ different classification tasks, each of which uses the corresponding heat sub-level set persistence 
diagrams as inputs.
The results we obtain using adaptive template functions are contrasted with \cite{preprint}. The results from this experiment highlight how CDER provides a more reliable method for obtaining adaptive templates when compared to GMM and HDBSCAN. Furthermore, when we select the heat kernel signature corresponding to the 6-th frequency, CDER adaptive templates generate a classification model with accuracy on par with the best results in \cite{preprint} and \cite{Reininghaus}. 
When compared to the non-adaptive (tent) templates of \cite{preprint}, our classification results are superior.

Finally, we present results for a protein classification problem
on the publicly available Protein Classification Benchmark data set PCB00019   \cite{protein_collection}. This data set contains spatial information for $1,357$ proteins as well as $55$ distinct supervised classification tasks. The results on \cite{protein_class} are used as a benchmark,
since they also use persistence diagrams, but the 
extracted features are hand-crafted  
to reflect chemical/physical properties of interest. In this experiment, adaptive template functions improve  the average classification accuracy reported in \cite{protein_class} from $82\%$ to around
$98\%$.

\section{Approximating Functions on Persistence Diagrams}
The goal of this section is to provide the theoretical 
framework in which template functions are used 
as a means to approximating continuous functions 
on the space of  persistence diagrams. 

A \textbf{persistence diagram}  is a pair $S_\mu = (S,\mu)$ where
\begin{enumerate}
    \item $S\subset \mathbb{W} := \set{(x_1,x_2)\in\mathbb{R}^2 \;\vert\; 0\leq x_1<x_2}$ is such that for any $\epsilon > 0$, the set 
    \[
    u_\epsilon(S) = \set{(x_1,x_2)\in S \;\vert\; pers(x_1,x_2) :=  x_2 - x_1>\epsilon}
    \] is finite. 
    \item $\mu : S\rightarrow \mathbb{N}$ is an arbitrary  (multiplicity) function.
\end{enumerate}

We will denote the elements of $S_\mu$ as $(x,m)$, 
where $x\in S$ and $1\leq m \leq \mu(x)$ is an integer.
Let $\mathcal{D}$ denote the \textbf{set of persistence diagrams};
this set comes equipped with a metric, the bottleneck 
distance, which we describe next.

\begin{definition}
A \textbf{partial matching} $M$ between persistence diagrams $S_\mu$ and $T_\alpha$ is a bijection between a subset of $S_\mu$ and a subset of $T_\alpha$; i.e., $M: S_\mu'\subset S_\mu \rightarrow T_\alpha'\subset T_\alpha$. If $(y,n) = M(x, m)$ we say that $(x,m)$ is \textbf{matched} with $(y,n)$. If $(z, k) 
\notin S_\mu'$ or $(z, k) \notin  T_\alpha'$ we call it \textbf{unmatched}.
\end{definition}

\begin{definition}
Given $\delta>0$, a partial matching $M$ is called a \textbf{$\delta$-matching} if
\begin{enumerate}
    \item If $(x,m) \in S_\mu'$ is matched with $(y, n)$, then \[\|x-y\|_\infty \leq \delta\] 
    \item If $(z,m)\in S_\mu \cup T_\alpha$ is unmatched, then $pers(x) \leq 2\delta$. 
\end{enumerate}
\end{definition}

\begin{definition} 
The \textbf{bottleneck distance} $d_B :\mathcal{D}\times\mathcal{D}\rightarrow \mathbb{R}^+$ is defined as \[d_B(D_1, D_2) = \inf \{ \delta > 0 : M:D_1\rightarrow D_2 \text{ is a $\delta$-matching} \}. \]\end{definition}

In \cite{Cohen-Steiner2007} it is shown  that $d_B$ defines a metric on $\mathcal{D}$ and that $\mathcal{D}$ is the metric completion  of $\mathcal{D}_0 := \set{(S,\mu)\in\mathcal{D} \;\vert\; S \text{ is finite}}$ with respect to $d_B$. 
In \cite{preprint} the authors present a complete characterization of (relatively) compact subsets of $(\mathcal{D}, d_B)$. In particular, one can prove that compact subsets of $(\mathcal{D}, d_B)$ have empty interiors (see Theorem $13$ in \cite{preprint}). 
This implies that $(\mathcal{D}, d_B)$ is not locally compact and therefore the compact-open topology on $C(\mathcal{D}, \R)$, the space of real-valued continuous 
functions on $\mathcal{D}$,  is not metrizable.

\subsection{Template functions}
\begin{lemma}
Let $C_c(\W)$ denote the set of real-valued continuous 
functions on $\W$ with compact support. 
For each $f\in C_c(\W)$, the function $\nu_f:\mathcal{D}\rightarrow\R$ defined as 
\[
\displaystyle \nu_f(S_\mu) = \sum_{x\in S}\mu(x)f(x)
\] is continuous.
\end{lemma}
\begin{proof}
See Lemma $23$ in \cite{preprint}.
\end{proof}

\begin{definition}
A \textbf{coordinate system} for $\mathcal{D}$ is a collection $\mathcal{F}\subset C(\mathcal{D}, \R)$ with the following property:  for any two  distinct $D,D' \in \mathcal{D}$,   there exists $F\in \mathcal{F}$ such that $F(D)\neq F(D')$.
\end{definition}

\begin{remark}
$\mathcal{F} = C(\mathcal{D}, \R)$ is itself a coordinate system, 
but at least for computational purposes, it is too large 
to be of algorithmic use.
\end{remark}

\begin{definition}
A \textbf{template system} for $\mathcal{D}$ is a set $T\subset C_c(\W)$ such that $\set{ \nu_f : f\in T }$ is a coordinate system for $\mathcal{D}$. The elements of $T$ are called template functions.
\end{definition}

The main utility of template systems 
is that they can be used to construct
dense subsets of $C(\mathcal{D}, \R)$
with respect to the compact-open topology.
In this topology, which is not metrizable as we mentioned 
above,  two 
functions are deemed  to be nearby if their values 
on compact sets are similar. 
Since the space of persistence diagrams 
is rather large and complicated, 
such comparisons (weaker than $L^2$ or $\|\cdot\|_\infty$) are desirable. 

\begin{theorem}\label{thm:regression}
Let $T$ be a template system for $\mathcal{D}$, let $\mathcal{C} \subset \mathcal{D}$ be compact,  and
let $F:\mathcal{C}\rightarrow \R$ be continuous. Then, for any $\epsilon > 0 $ there exist $N\in \N$, a polynomial $p\in \R[t_1, \dots, t_N]$, and template functions $f_1, \dots, f_N \in T$ so that  
\begin{equation}\label{eqn:reg}
    \vert p(\nu_{f_1}(D), \dots, \nu_{f_N}(D)) - F(D) \vert <\epsilon
\end{equation} for all $D\in\mathcal{C}$.

\end{theorem}
\begin{proof}
See Theorem $29$ in \cite{preprint}.
\end{proof}

In other words,

\begin{corollary}
Let $T\subset C_c(\mathbb{W})$ be a template 
system for $\mathcal{D}$. Then,
the collection of functions of the form
\[
\begin{array}{ccl}
\mathcal{D} & \longrightarrow &\R \\ 
D &\mapsto &p(\nu_{f_1}(D), \dots, \nu_{f_N}(D))
\end{array}
\]
where $N\in \N$, $p\in \R[t_1,\ldots, t_N]$,   and $f_n \in T$, 
is dense in $C(\mathcal{D},\R)$ 
with respect to the compact-open topology.
\end{corollary}

The problem of
constructing template systems 
of reasonable size (e.g., countable)
is addressed by the following 
theorem.

\begin{theorem}
Let $f\in C_c(\mathbb{W})$, $n\in \mathbb{N}$,
$\mathbf{m}\in \mathbb{Z}^2$ 
and define
\[
f_{n,\mathbf{m}}(x) = 
f\left(nx + \frac{\mathbf{m}}{n}\right)
\]
If $f$ is nonzero, then 
$
T = \{f_{n,\mathbf{m}} | 
n\in \mathbb{N}, \; \mathbf{m}\in \mathbb{Z}^2\}\cap C_c(\mathbb{W})
$ is a template system for $\mathcal{D}$.
\end{theorem}
\begin{proof}
See Theorem 30 in \cite{preprint}.
\end{proof}

The goal of this paper is to investigate 
and identify
data-driven methodologies for selecting the 
support of an
initial template function $f$, 
as well as its most relevant re-scaled translates 
$f_{n,\mathbf{m}}$,  
so that the scalar  features 
$\nu_{f_{n,\mathbf{m}}}$ can be used 
successfully in learning  
problems on persistence diagrams.

\section{Adaptive template systems}
In \cite{preprint}    two different template systems are suggested: tent functions and interpolating polynomials.
Specifically, let 
\[
\widetilde{\W} = \set{(x,y)\,\vert\, x\in\R , y\in\R_{>0}}
\]
be the conversion of $\W$ to the  birth-lifetime plane. This conversion is defined by $(a,b)\in\W \mapsto (a,b-a)\in\widetilde{\W}$. 
The \textbf{template system of tent functions} in the birth-lifetime plane is defined as follows. Let $\vect{a}=(a,b)\in\widetilde{\W}$ and $0<\delta<b$, then 
\[
g_{\vect{a},\delta}(x,y) = 
\max 
\left\{ 
1 - \frac{1}{\delta}
\max\{\abs{x-a},\abs{y-b}\}, \; 0  
\right\}
\]
 
In a similar manner one can define a \textbf{template system of interpolating polynomials}. 
Given $\set{(a_i,b_j)}_{i,j}\subset \widetilde{\W}$ and $\set{c_{i,j}}_{i,j}\subset\R$, one can use Lagrange interpolating polynomials to construct a function $f$ such that $f(a_i,b_j) = c_{i,j}$.
 
In general these two approaches require the user to input the meshes used in defining the template systems. By construction, such meshes define the support of the template functions. One shortcoming of this procedure, when applied to \Cref{thm:regression}, is that without prior knowledge about the compact set $\mathcal{C}\subset\mathcal{D}$ the number of template functions that carry no information relevant to the problem can be high. This drawback is illustrated in \Cref{fig:mesh_pdgm}. 

\begin{figure}[!htb]
    \centering
    \includegraphics[width=0.8\textwidth]{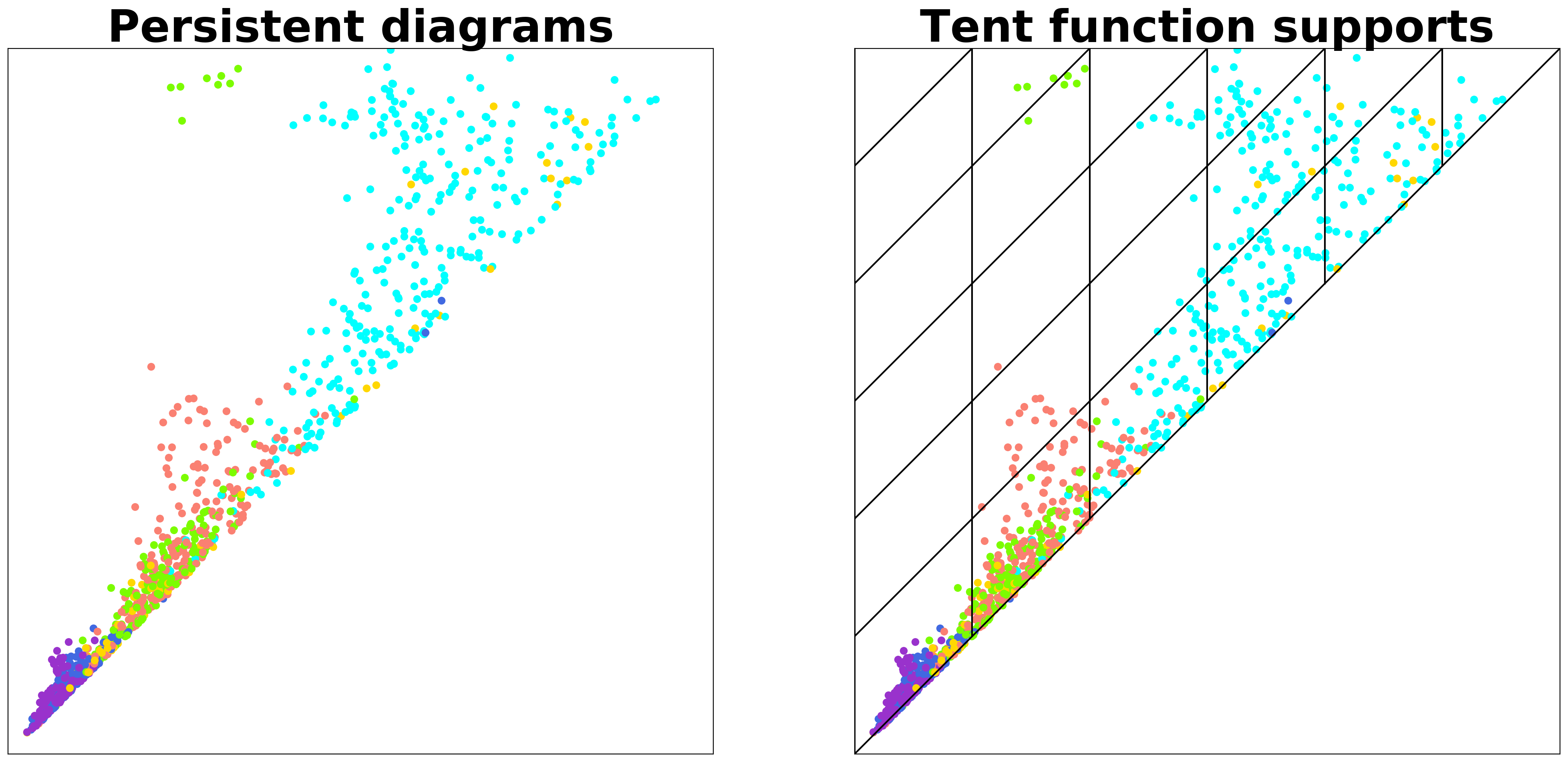}
    \caption{\textbf{Left:} Collection of persistent diagrams colored by class. \textbf{Right:} Mesh covering the collection on the left.}
    \label{fig:mesh_pdgm}
\end{figure}

The main goal of this paper is to present a methodology to define the template system used in \Cref{thm:regression} that incorporates the  prior information we have about the particular learning task. Such methodology is what we refer to as \textbf{adaptive template functions}.

\begin{figure}[!htb]
    \centering
    \includegraphics[width=0.8\textwidth]{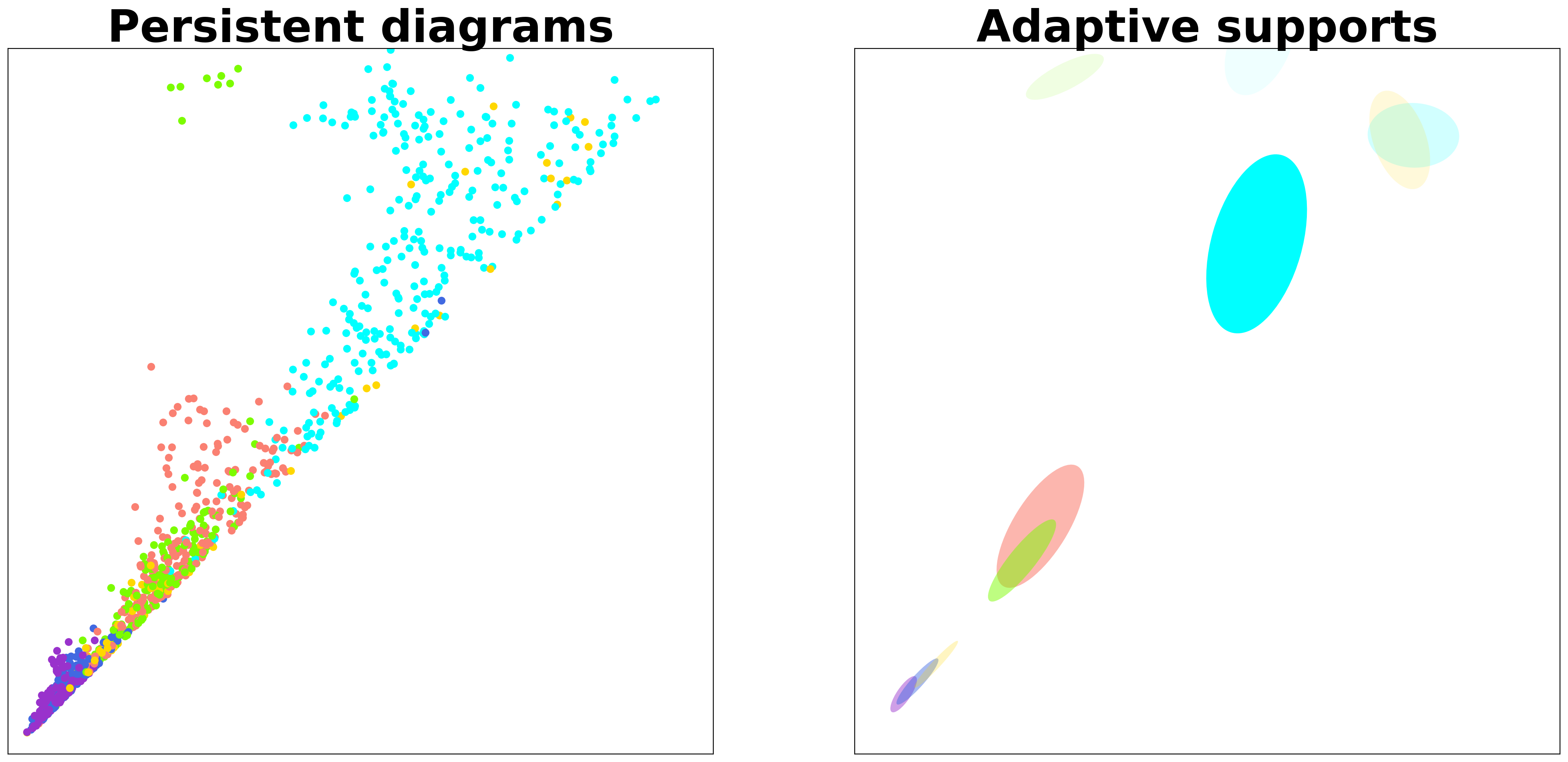}
    \caption{\textbf{Left:} Collection of persistent diagrams colored by class. \textbf{Right:} collection of open balls as supports for template functions.}
    \label{fig:cover_pdgm}
\end{figure}

Our approach to defining adaptive template functions is to first identify a collection of open ellipses  in $\widetilde{\W}$ or $\W$  as in \Cref{fig:cover_pdgm}. Each ellipse in this collection will be the support for a template function defined in the following manner. Let $A\in M_{2\times 2}(\R)$ represent the quadratic form in two variables corresponding to an ellipse in the collection, and let $\vect{x}\in \W$ be its center. Then, the associated template function is   
\[
f_A(\vect{z}) 
= 
\begin{cases}
1 - (\vect{z} - \vect{x})^* A (\vect{z} - \vect{x}) &, \; (\vect{z} - \vect{x})^* A (\vect{z} - \vect{x}) < 1 \\
0 
&,\; (\vect{z} - \vect{x})^* A (\vect{z} - \vect{x}) \geq 1
\end{cases}
\]

To obtain the collection of ellipses ($A$) mentioned above we will use and compare three different approaches; namely, Cover-Tree Entropy Reduction (CDER, \cite{cder_paper}), Gaussian Mixture Models (GMM) and Hierarchical Density-Based Spatial Clustering of Applications with Noise (HDBSCAN).
We now provide a brief description of each method.

\subsection{ Cover-Tree  Entropy  Reduction - CDER}\label{subsec:cder}
The main objective of this algorithm is to find a partial cover tree for a collection of labeled point clouds in $\R^n$. CDER searches for the cover tree in convex regions that are likely to have minimum local entropy. In this section we will explain the notion of entropy used by CDER, as it is relevant to explaining some of the results in \Cref{sec:examples}.

Let $\chi = \set{X_1,\dots,X_N}$ be a collection of point clouds $X_i\subset\mathbb{R}^d$ ---
which in our case will be persistence diagrams ---
and  define $\displaystyle \underline{\chi} := \bigsqcup_{i=1}^N X_i$.
We also have a labeling map at the level of point clouds
$$
\lambda : \chi\rightarrow\set{1,\dots,L} := \mathcal{L}.
$$

Notice that we have a natural map $ind : \underline{\chi}\rightarrow \chi$ given by $ind(x) = X_i$ if $x\in X_i$. This allow us to define $$\underline{\chi}_l := (\lambda\circ ind)^{-1}(l),$$ the set of all point in a point cloud labeled $l\in\mathcal{L}$.

Now we will assign weights to each point $x\in\underline{\chi}$ in the following manner:

\begin{enumerate}
	\item Each label is equally likely among the data and $\underline{\chi}$ has a total weight of $1$.  Thus each label $l\in\mathcal{L}$ has an allocated weight of $1/L$.
	\item Now consider $$\lambda^{-1}(l) = \set{\text{All point clouds with label } l},$$ we assume   again that each point cloud in $\lambda^{-1}(l)$ is equally likely, 
	so each $X_i\in\lambda^{-1}(l)$ has an allocated weight of $1/(N_l L)$, where $N_l = \norm{\lambda^{-1}(l)}$.
	\item Finally each point in $x\in X_i$ is equally likely, 
	and thus 
	\[w(x) = \frac{1}{\norm{X_i} N_l L}\]
\end{enumerate}

\begin{definition} 
Let $\chi = \set{X_1,\dots,X_N \;\vert\; X_i\subset \mathbb{R}^d}$ 
be a collection of point clouds
together with a label map $\lambda : \chi\rightarrow\set{1,\dots,L} := \mathcal{L}$ and a weight function $w:\underline{\chi}\rightarrow \mathbb{R}$. 
For any convex and compact set $\Omega \subset \mathbb{R}^d$ we define the following quantities:
\begin{enumerate}
	\item 
	The \textit{total weight of $l \in \mathcal{L}$} in $\Omega$ $$\displaystyle w_l(\Omega) = \sum\set{w(x) \;\vert\; x\in\Omega\cap\underline{\chi}_l}.$$ 
	\item the total weight of $\Omega$ $$W(\Omega) = \sum_{l\in\mathcal{L}} w_l(\Omega)$$
\end{enumerate}
\end{definition}

Using the weights assigned above we can interpret $\displaystyle \frac{w_l(\Omega)}{W(\Omega)}$ as the probability of a point in $\Omega$ to have label $l$. 
This leads to the following definition.

\begin{definition} Let  $\chi$, $\lambda $, and  $w$ be as before. 
For any convex and compact set $\Omega \subset \mathbb{R}^d$ its \textit{entropy} is defined by $$S(\Omega) = -\sum_{l\in\mathcal{L}} \frac{w_l(\Omega)}{W(\Omega)} \log_L \left( \frac{w_l(\Omega)}{W(\Omega)} \right).$$
\end{definition}

This definition is borrowed from   information theory. 
Notice that if all $w_l(\Omega)$ are roughly the same, then $S(\Omega) \approx 1$. But, if for example, $\Omega$ only contains points with a single label then we must have $S(\Omega) = 0$.

\begin{remark}\label{rem:entropy}

Generally, we would expect the entropy to decrease as we select smaller subsets of $\Omega$. 
However,  this is not always the case. 

\begin{figure}[!htb]
\includegraphics[width=0.8\textwidth]{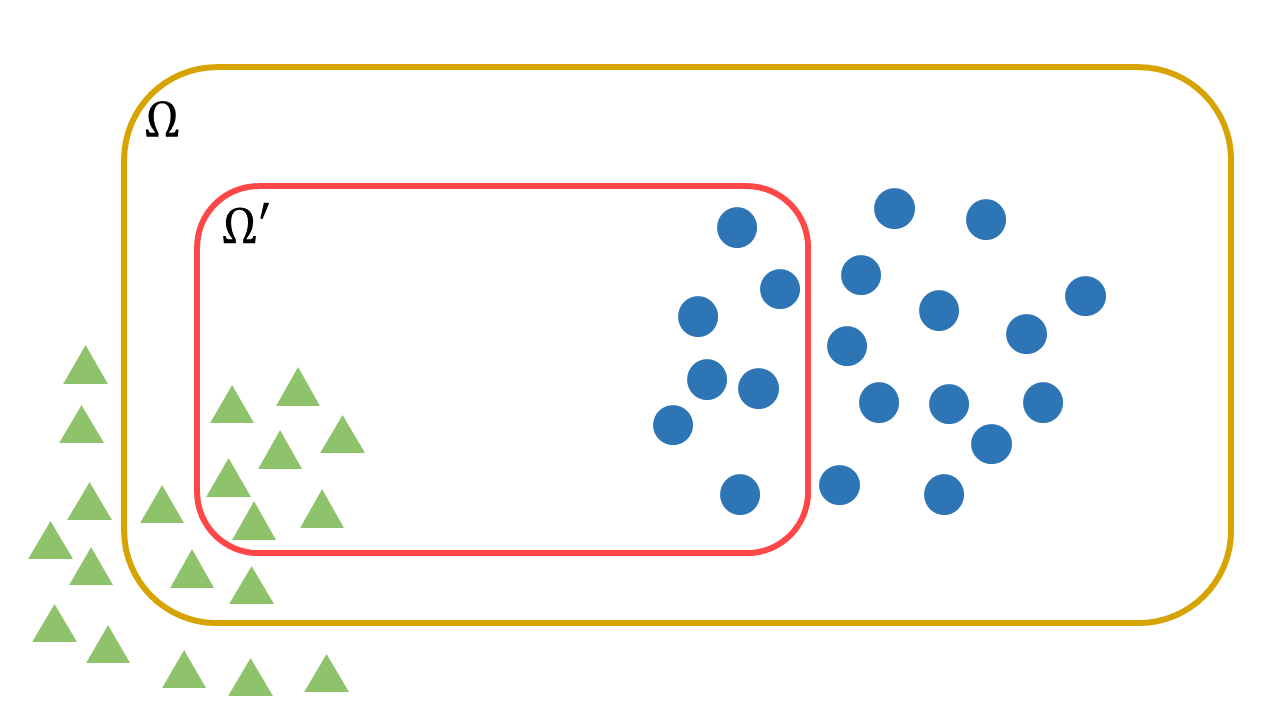}
\caption{$\Omega'\subset\Omega$ but $S(\Omega) < S(\Omega')$.}\label{fig:entropy}
\end{figure}

For instance, consider point clouds $X_1$ and $X_2$ such that $\abs{X_1}=10$ and $\abs{X_2}=20$, each point cloud with a different label $\set{0,1}$, and compact sets $\Omega, \Omega'$ as show in \Cref{fig:entropy}.
We can easily see that $\omega_0(\Omega') =  \omega_1(\Omega') = 7/(20\cdot 20\cdot 2)$, so $W(\Omega') = 7 / (20\cdot 20\cdot 2) + 7/(20\cdot 20\cdot 2) = 7/400$. Finally $S(\Omega') = 1$.
A similar calculation will show that $S(\Omega) \approx 0.9182$,
and thus $\Omega'\subset\Omega$ but $S(\Omega) < S(\Omega')$.
\end{remark}

\subsection{Gaussian Mixture Models - GMM}

This algorithm is an implementation of the Expectation-Maximization (EM) algorithm to fit Gaussian models to a collection of points. Recall that an EM algorithm is an iterative method to solve maximum likelihood estimation of parameters for a given model; in our case Gaussian models.

The EM algorithm iterates over two steps: an expectation step and a maximization step. The first step defines the expected value of the log likelihood function using a given set of parameters for the model. The maximization step finds a new set of parameters that maximizes the previously described expected value.

\subsection{Hierarchical Density-Based Spatial Clustering of Applications with Noise - HDBSCAN}

HDBSCAN \cite{hdbscan} is a hierarichal clustering algorithm  extending  BDSCAN. The latter, finds clusters as the connected components of a graph. The vertices of this graph are the elements in the data set after removing  ``noise points'' and the adjacency is defined by a user-provided  parameter $\epsilon$.

HDBSCAN constructs a hierarchical collection of DBSCAN solutions by changing the parameter used to compute the adjacency in the DBSCAN algorithm. Once this hierarchy of solutions is obtained, the algorithm extracts from its hierarchy dendrogram a sumarized collection of significant clusters.

\section{Experimental Results}\label{sec:examples}
Thus far we have developed a methodology for deriving adaptive template systems from labeled persistence diagrams.
The main idea is to use algorithms such as 
CDER, GMM and HDBSCAN to identify the supports of the 
functions in the template system.
We will use these adaptive templates to produce feature vectors in order to solve supervised classification problems. In this section we present three examples of supervised learning, where adaptive template functions yield featurizations that improve the classification accuracy or robustness of the classification model.
Our implementation of adaptive template systems with CDER, GMM and HDBSCAN 
as well as the scripts to replicate all the results in this section can be found in the associated GitHub repository\footnote{https://github.com/lucho8908/adaptive\_template\_systems.git}.

\subsection{Manifolds}

\begin{figure}[!htb]
    \centering
    \includegraphics[width=0.8\textwidth]{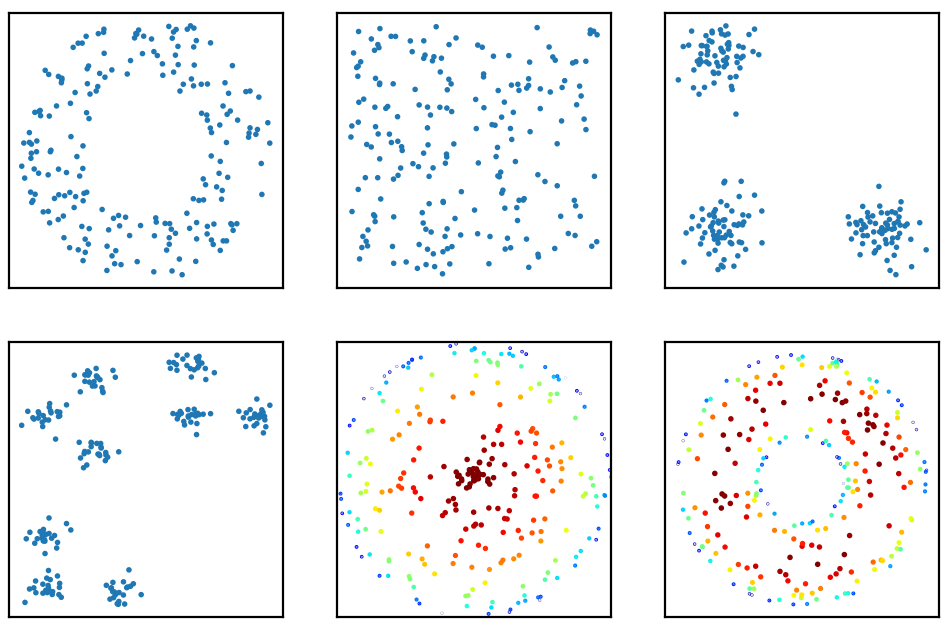}
    \caption{Example of the $6$ manifolds, from top left to bottom right we have: annulus, cube, $3$ clusters, $3$ cluster of $3$ clusters, $S^2$ (projected on the $xy$-plane) and torus (projected on the $xy$-plane).}
    \label{fig:manifolds}
\end{figure}

For this example we revisit an experiment presented in \cite{preprint} and \cite{persistent_images}. We generated point clouds sampled from different manifolds in $\R^2$ or $\R^3$. Each point cloud has $200$ points and the manifolds considered are the following: an \textbf{annulus} with inner radius 1 and outer radius 2 centered at $(0,0)$, \textbf{3 clusters} of points drawn from normal distributions with means $(0,0)$, $(0,2)$ and $(2,0)$ all with standard deviation $0.05$, \textbf{3 cluster of 3 clusters} of points drawn from normal distributions with standard deviation $0.05$ and means $(0,0)$, $(0,1.5)$, $(1.5,0)$, $(0,4)$, $(1,3)$, $(1,5)$, $(3,4)$, $(3,55)$ and $(4.5,4)$, \textbf{cube} defined as $[0,1]^2\subset\R^2$, \textbf{torus} obtained from rotating a circle of radius $1$ centered at $(2,0)$ on the $xz$-plane around the $z$-axis and \textbf{sphere} $S^2\subset\R^3$ with uniform noise in $[-0.05,0.05]$ on the normal direction.

We used CDER, GMM and HDBSCAN to generate the supports of the functions  that form our template systems. We reserved $33\%$ of the data for testing,  and trained a kernel ridge regression model on the remaining data 
($\%67$). 
In addition, we investigate the effect of increasing the number of point clouds sampled from each manifold.

\begin{table*}[!htb]
\centering
\caption{Manifold classification: Each row corresponds to the number of samples taken from each manifold. The first two columns show the best results reported in \cite{preprint}.}
\label{tab:manifold_reslts}
\begin{adjustbox}{width=1.2\textwidth,center=\textwidth}
\small
\begin{tabular}{| c | cc | cc | cc | cc |} \hline
             &                &                & \textbf{CDER}                            &                                          & \textbf{GMM}     &                  & \textbf{HDBSCAN} &                  \\
             & \textbf{Train} & \textbf{Test}  & \textbf{Train}                           & \textbf{Test}                            & \textbf{Train}   & \textbf{Test}    & \textbf{Train}   & \textbf{Test}    \\ \hline
\textbf{10}  & 0.99 $\pm$ 0.9 & 0.96 $\pm$ 3.2 & \cellcolor[HTML]{90EE90}0.99 $\pm$ 0.001 & \cellcolor[HTML]{90EE90}0.98 $\pm$ 0.034 & 0.99 $\pm$ 0.001 & 0.91 $\pm$ 0.075 & 1.00 $\pm$ 0.000 & 0.90 $\pm$ 0.092 \\
\textbf{25}  & 0.99 $\pm$ 0.3 & 0.99 $\pm$ 1.0 & \cellcolor[HTML]{90EE90}0.99 $\pm$ 0.001 & \cellcolor[HTML]{90EE90}0.99 $\pm$ 0.001 & 0.99 $\pm$ 0.004 & 0.99 $\pm$ 0.009 & 1.00 $\pm$ 0.000 & 0.89 $\pm$ 0.031 \\
\textbf{50}  & 1.00 $\pm$ 0.0 & 0.99 $\pm$ 0.9 & \cellcolor[HTML]{90EE90}0.99 $\pm$ 0.001 & \cellcolor[HTML]{90EE90}0.99 $\pm$ 0.001 & 0.99 $\pm$ 0.002 & 0.99 $\pm$ 0.008 & 1.00 $\pm$ 0.000 & 0.95 $\pm$ 0.003 \\
\textbf{100} & 0.99 $\pm$ 0.1 & 0.99 $\pm$ 0.4 & \cellcolor[HTML]{90EE90}0.99 $\pm$ 0.001 & \cellcolor[HTML]{90EE90}0.99 $\pm$ 0.001 & 0.99 $\pm$ 0.004 & 0.99 $\pm$ 0.005 & 1.00 $\pm$ 0.000 & 0.97 $\pm$ 0.011 \\
\textbf{200} & 0.99 $\pm$ 0.1 & 0.99 $\pm$ 0.3 & \cellcolor[HTML]{90EE90}0.99 $\pm$ 0.002 & \cellcolor[HTML]{90EE90}0.99 $\pm$ 0.005 & 0.99 $\pm$ 0.002 & 0.99 $\pm$ 0.003 & 1.00 $\pm$ 0.000 & 0.98 $\pm$ 0.005 \\ \hline
\end{tabular}
\end{adjustbox}
\end{table*}

In \Cref{tab:manifold_reslts} (see \cpageref{tab:manifold_reslts}) we present the mean accuracy of the model on both the training and testing data after averaging the results over $10$ experiments for each sampling size. Furthermore, \Cref{tab:manifold_reslts} contains the results obtained from using CDER, GMM and HDBSCAN to find the adaptive templates as well as the best results presented by \cite{preprint} on the same classification problems.

The first important feature to highlight is that for all the different sampling sizes the adaptive templates accuracy results show a smaller standard deviation than the results reported in \cite{preprint}. At the same time, the mean accuracy of our methodology is comparable with the state of the art results (in \cite{preprint}). 

It is worth mentioning that across all the different methods used in this work to obtain adaptive template systems, CDER provides the most stable results. This is specially significant for the smaller sample size (the first row in \Cref{tab:manifold_reslts}). 

\subsection{Shape data}\label{subsec:shape_data}

In this example we consider the synthetic SHREC 2014 data set \cite{Pickup:2014}, 
of which some instances are shown in Figure \ref{fig:shrecExamples}. We compare our result to the methods reported in \cite{preprint} by extracting features using adaptive template systems for the same data from \cite{Reininghaus} and \cite{preprint}. In \cite{Reininghaus} the authors defined a function on each mesh using a heat kernel signature, \cite{sun}, for $10$ parameters and computed persistent diagrams for dimensions $0$ and $1$.

\begin{figure}[!htb]
    \centering
    \includegraphics[width=.8\textwidth]{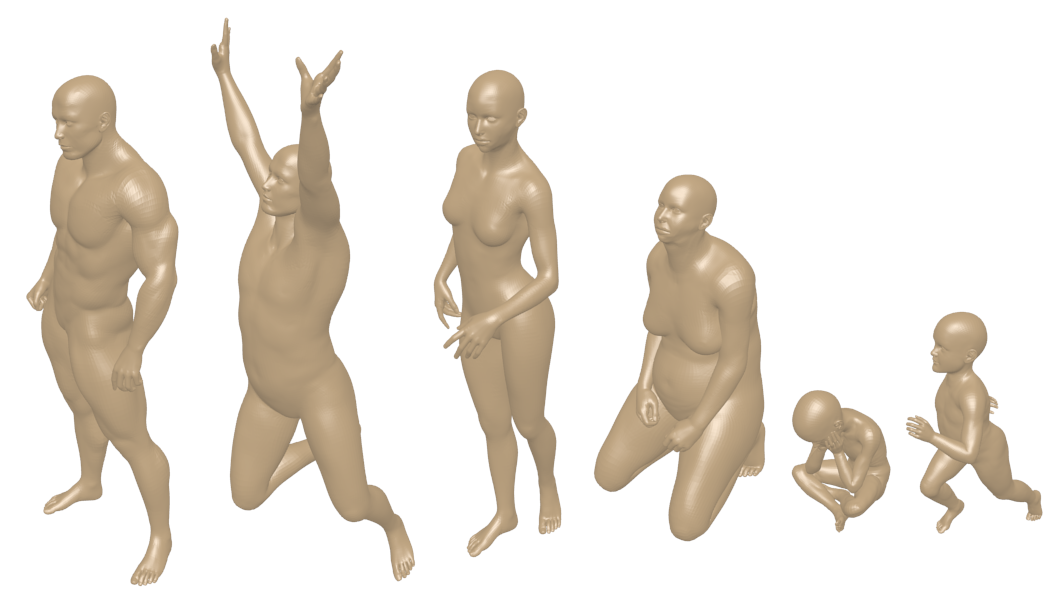}
    \caption{Examples of shapes and poses in the SHREC synthetic data set.}
    \label{fig:shrecExamples}
\end{figure}

For each one of the $10$ parameter values we have $300$ pairs of persistence diagrams and the goal of the problem is to predict the human model. The models correspond to $5$ different poses for people labeled as male, female and child; giving us a total of $15$ labels. Lets us remark that each one of the $10$ parameters yields a different classification problem. 

\begin{table}[!htb]
\centering
\caption{Shape classification: Portion out of the $10$ problems, for which each adaptive template system yields the best classification results. The cells with a dash (-) indicate that no computations were carried out.}
\label{tab:shape_results}
\begin{tabular}{| c | c | c | c |} \hline
              & \bf{Polynomial}                    & \bf{RBF}                           & \bf{Sigmoid}            \\ \hline
\bf{Kernel}   & \cellcolor[HTML]{90EE90}CDER = 0.6 & \cellcolor[HTML]{90EE90}CDER = 0.4 & CDER = 0.2                            \\
\bf{Ridge}    & GMM = 0.3                          & \cellcolor[HTML]{90EE90}GMM = 0.4  & \cellcolor[HTML]{90EE90}GMM = 0.4     \\
              & HDBSCAN = 0.1                      & HDBSCAN = 0.1                      & \cellcolor[HTML]{90EE90}HDBSCAN = 0.4 \\ \hline
              & \cellcolor[HTML]{90EE90}CDER = 0.7 & \cellcolor[HTML]{90EE90}CDER = 0.7 & \cellcolor[HTML]{90EE90}CDER = 0.8    \\
\bf{SVM}      & GMM = 0.3                          & GMM = 0.2                          & GMM = 0.2                             \\
              & HDBSCAN = 0.0                      & HDBSCAN = 0.0                      & HDBSCAN = 0.0                         \\ \hline
\bf{Random}   & \cellcolor[HTML]{90EE90}CDER = 0.6 & -                                  & -                                     \\
\bf{Forest}   & GMM = 0.3                          & -                                  & -                                     \\
              & HDBSCAN = 0.1                      & -                                  & -                                     \\ \hline
\end{tabular}
\end{table}

We considered CDER, GMM and HDBSCAN as methods to obtain adaptive template systems. For each one of these template systems we used three different kernel methods to solve the classification problems, namely, kernel ridge regression, kernel support vector machines (SVM) and random forest. Finally, three different kernels were examined, polynomial, radial basis function (RBF) and sigmoid kernels. 

\Cref{tab:shape_results} shows the portion, out of the $10$ problems, for which a given adaptive template system yields the best classification results. 
For instance, 
the entry corresponding to ridge regression with a polynomial kernel (first row and first column). 
shows that the CDER template system yields the best classification accuracy in $6$ our of $10$ problems, GMM is the best for $3$ out of $10$ and HDBSCAN is superior in $1$ out of the $10$ problems.

With this interpretation of \Cref{tab:shape_results} in mind, we can see that CDER adaptive template systems yield more accurate classification results than GMM or HDBSCAN templates. This holds true for all but one, of the kernel and kernel method combinations explored in this experiment.

\begin{table*}[!htb]
\centering
\caption{Shape classification: Classification accuracy for adaptive templates (ours), tent functions and interpolating polynomials \cite{preprint}}
\label{tab:shape_results_best}
\begin{adjustbox}{width=1.2\textwidth,center=\textwidth}
\small
\begin{tabular}{|c|c c|c c|c c c c|}
\hline
            & \multicolumn{2}{|c|}{\textbf{Global  features}}         & \multicolumn{6}{|c|}{\textbf{Local features}}                                                                                                                                                                                                               \\ \hline
            & \multicolumn{2}{|c|}{\textbf{Interpolating Polynomials}} & \multicolumn{2}{|c|}{\textbf{Tent functions}}    & \multicolumn{4}{|c|}{\textbf{Adaptive templates}}                                                                                                                                                       \\ 
\textbf{Freq.}       & \textbf{Train}                         & \textbf{Test}                           & \textbf{Train}                        & \textbf{Test}                        & \textbf{Train}                        & \textbf{Test}                         &                              &                             \\ \hline
\textbf{1 } & 0.99 $\pm$ 0.3                                  &  0.90 $\pm$ 5.3                         & 0.08$\pm$0.30                         & 0.03$\pm$0.5                         & 0.79$\pm$0.01                         & 0.73$\pm$0.02                         & GMM                          & Kernel Ridge                \\ 
\textbf{2 } & 1.00 $\pm$ 0.0                                  &  0.95 $\pm$ 2.4                         & 0.08$\pm$0.40                         & 0.03$\pm$1.00                        & 0.84$\pm$0.00                         & 0.8$\pm$0.03                          & GMM                          &                             \\
\textbf{3 } & 0.99 $\pm$ 0.5                                  &  0.90 $\pm$ 2.0                         & 0.80$\pm$1.3                          & 0.44$\pm$4.3                         & 0.7$\pm$0.01                          & 0.66$\pm$0.03                         & HDBSCAN                      &                             \\
\textbf{4 } & 0.98 $\pm$ 0.9                                  &  0.84 $\pm$ 3.9                         & 0.89$\pm$1.5                          & 0.69$\pm$4.9                         & 0.7$\pm$0.01                          & 0.67$\pm$0.03                         & GMM                          &                             \\
\textbf{5 } & 0.99 $\pm$ 0.4                                  &  0.93 $\pm$ 2.2                         & 0.76$\pm$2.7                          & 0.58$\pm$7.9                         & 0.78$\pm$0.04                         & \cellcolor[HTML]{90EE90}0.76$\pm$0.08 & \cellcolor[HTML]{90EE90}CDER &                             \\ \hline
\textbf{6 } & \cellcolor[HTML]{90EE90}0.98 $\pm$ 0.5          &  \cellcolor[HTML]{90EE90}0.92 $\pm$ 1.8 & \cellcolor[HTML]{90EE90}0.96$\pm$0.67 & \cellcolor[HTML]{90EE90}0.89$\pm$1.7 & \cellcolor[HTML]{90EE90}0.97$\pm$0.01 & \cellcolor[HTML]{90EE90}0.92$\pm$0.03 & \cellcolor[HTML]{90EE90}CDER & \cellcolor[HTML]{90EE90}SVM \\
\textbf{7 } & 0.99 $\pm$ 0.4                                  &  0.95 $\pm$ 1.4                         & 0.98$\pm$0.60                         & 0.94$\pm$2.5                         & 0.96$\pm$0.02                         & \cellcolor[HTML]{90EE90}0.94$\pm$0.05 & \cellcolor[HTML]{90EE90}CDER &                             \\ \hline
\textbf{8 } & 0.99 $\pm$ 0.4                                  &  0.94 $\pm$ 2.2                         & 0.91$\pm$1.2                          & 0.89$\pm$3.3                         & 1$\pm$0.00                            & \cellcolor[HTML]{90EE90}0.88$\pm$0.05 & \cellcolor[HTML]{90EE90}CDER & Random Forest               \\
\textbf{9 } & 0.98 $\pm$ 1.3                                  &  0.92 $\pm$ 2.1                         & 0.64$\pm$2.3                          & 0.53$\pm$3.8                         & 1$\pm$0.00                            & \cellcolor[HTML]{90EE90}0.88$\pm$0.03 & \cellcolor[HTML]{90EE90}CDER &                             \\
\textbf{10} & 0.97 $\pm$ 1.1                                  &  0.89 $\pm$ 4.6                         & 0.27$\pm$3.4                          & 0.18$\pm$5.6                         & 1$\pm$0.00                            & \cellcolor[HTML]{90EE90}0.9$\pm$0.08  & \cellcolor[HTML]{90EE90}CDER &                             \\ \hline
\end{tabular}
\end{adjustbox}
\end{table*}

Having stablished how CDER, GMM and HDBSCAN compare across different kernel methods and kernels,
next we 
compare our classification results with the state of the art. \Cref{tab:shape_results_best} contains  our classification accuracy on the training and testing shape data, as well as the results obtained using
the tent functions and interpolating polynomials from \cite{preprint}.

The first important feature to remark is that for $7$ out of the $10$ problems, the model obtained from adaptive coordinates has less overfitting than 
interpolating polynomials. Meaning that for most of the classification problems studied in this example, adaptive template systems provide a more robust classification model. 
An additional argument in favor of the robustness of our approach is 
that the standard deviation of all the models using adaptive templates is smaller than those presented in \cite{preprint}.
Moreover, when comparing adaptive template systems with tent functions  (see \Cref{tab:shape_results_best}) we attain better classification results on the testing set across all the problems presented. 
This highlights the benefits 
of adaptive versus non-adaptive local template systems.

It is pertinent to mention that the results in \Cref{tab:shape_results_best} correspond to a specific selection of parameters for each kernel and regularization in each classification method. In fact, from our methodology we can find a combination of kernel parameters and regularization that yields higher classification accuracy in the training set. But, for such models the overfitting issues are more noticeable.

Finally, since the end goal of this problem is to solve the multiclass classification problem in the synthetic SHREC 2014 data set. We can select the problem corresponding to the frequency $6$ in the heat kernel signature (row $6$ in \cref{tab:shape_results_best}) as the one that gives us the best classification accuracy while minimizing the overfiting concern. Such result is obtained using a CDER template system and a regularized SVM method.

\subsection{Protein classification}\label{subsec:protein_classification}

We consider next the data set PCB00019 from the Protein Classification Benchmark Collection \cite{protein_collection}. This problem set contains $1,357$ proteins and $55$ classification tasks. Our results will be compared to those reported  in \cite{protein_class}.

\begin{figure}[!htb]
    \centering
    \includegraphics[width = 0.8\textwidth]{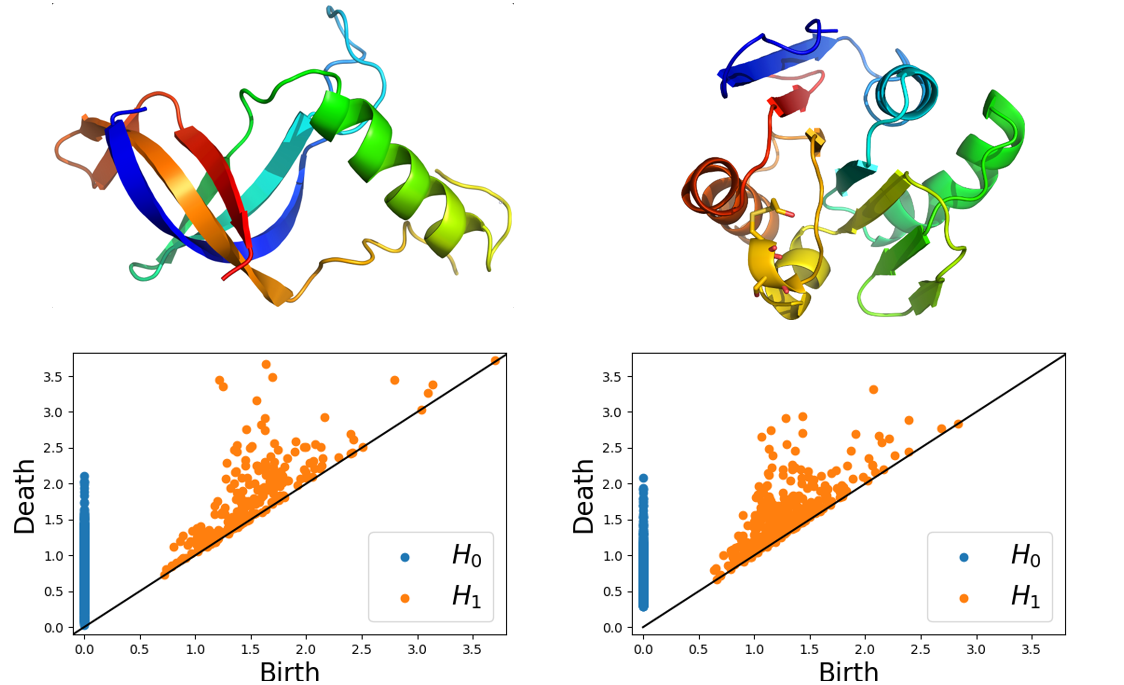}
\caption{Examples of data points in PCB00019 and their corresponding persistent diagrams. \textbf{Top row:} Protein domains from \cite{SCOPe}. \textbf{Bottom row:} Persistent diagrams in dimensions $0$and $1$.}
\label{tab:proteins_and_dgms}
\end{figure}

To compare our results with the ones in \cite{protein_class}, we report the average classification accuracy over the $55$ classification tasks in the data set PCB00019.

\begin{table}[!htb]
\centering
\caption{Protein classification: average classification accuracy for the $55$ tasks in the data set PCB00019.}
\label{tab:proteins}
\begin{tabular}{cc|c|c|}
\cline{3-4}
                                           &                                                  & \textbf{Train}                                                 & \textbf{Test}                                                  \\ \hline
\multicolumn{1}{|c|}{\textbf{}}            & {\color[HTML]{000000} \textbf{Polynomial}}       & \cellcolor[HTML]{90EE90}{\color[HTML]{000000} 0.90 $\pm$ 0.07} & \cellcolor[HTML]{90EE90}{\color[HTML]{000000} 0.98 $\pm$ 0.02} \\ \cline{2-4} 
\multicolumn{1}{|c|}{\textbf{CDER}}        & \textbf{RBF}                                     & 0.91 $\pm$ 0.06                                                & 0.97 $\pm$ 0.02                                                \\ \cline{2-4} 
\multicolumn{1}{|c|}{\textbf{}}            & \textbf{Sigmoid}                                 & \cellcolor[HTML]{90EE90}0.90 $\pm$ 0.07                        & \cellcolor[HTML]{90EE90}0.98 $\pm$ 0.02                        \\ \hline
\multicolumn{2}{|c|}{\textbf{Topological features in \cite{protein_class}}} & -                                                              & 0.82 $\pm$ ----                                                \\ \hline
\end{tabular}
\end{table}

\Cref{tab:proteins} shows the average classification accuracy and standard deviation for each 
of the classification tasks from PCB00019. 
Here we used a regularized kernel ridge regression with polynomial, RBF and sigmoid kernels. 
The last row in \Cref{tab:proteins} displays the average classification accuracy for the testing set reported in \cite{protein_class}. We note that the authors do not report standard deviation or average classification accuracy for the testing set. 

It is meaningful to remark that in \cite{protein_class} topological features are used to solve the classification problems. Those features were constructed using persistence diagrams as to reflect relevant properties of the proteins in the given data set.

\section{Discussion} 

This paper investigates the viability of utilizing data-driven  methodologies to localize features in persistence diagrams.
These features are used in subsequent 
supervised learning tasks
for classification problems 
where shape is an important feature.
We examined three different algorithms, CDER, GMM and HBDSCAN to produce adaptive template functions.
Through extensive testing 
with real and synthetic data sets, 
we demonstrate that
CDER provides a more robust collection of adaptive features while maintaining classification accuracy on par with the state of the art. 
In terms of time complexity, CDER 
also outperforms  GMM and HDBSCAN for the problems here considered.

\section*{Acknowledgment}

The authors gratefully thank Elizabeth Munch for providing the data set used in \cref{subsec:shape_data},   Kelin Xia for providing the data set used in \cref{subsec:protein_classification}, and   Lida Kanari for  providing a data set which did not make it into the final version of the paper. 
This work was partially supported by the NSF under grant DMS-1622301.

\bibliography{biblio}{}
\bibliographystyle{plain}
\end{document}